\documentclass{article}

\usepackage[ruled,vlined,linesnumbered]{algorithm2e}
\usepackage{enumitem}
\usepackage{wrapfig}
\usepackage{float}
\usepackage{subcaption}
\usepackage{amsmath, amssymb, amsthm}
\newtheorem{proposition}{Proposition}
\usepackage{graphicx}
\usepackage{booktabs}
\usepackage{tabularx}
\SetKwInput{KwRequire}{Require}
\SetKwComment{Comment}{$//$~}{}
\usepackage[final]{corl_2026} 

\title{Chunk-wise Model-based Policy Optimization}

\author{
  Weixi Song\textsuperscript{1,2,3}
  \qquad
  Donglin Wang\textsuperscript{3,$\dagger$}
  \qquad
  Yong-Lu Li\textsuperscript{2,4,$\dagger$}
  \\[1ex]
  \textsuperscript{1}Zhejiang University
  \qquad
  \textsuperscript{2}Shanghai Innovation Institute
  \qquad
  \textsuperscript{3}Westlake University
  \qquad
  \\
  \textsuperscript{4}Shanghai Jiao Tong University
  \qquad
  \textsuperscript{$\dagger$}Corresponding Author
}

\begin{document}
\maketitle


\begin{abstract}
    Model-based reinforcement learning has shown superior sample efficiency by leveraging the learned latent dynamics either in the form of derivative-free planning/imagined rollout training or dynamics gradient-based policy updates. While these approaches achieve comparatively fast learning, they all suffer, to varying degrees, from issues of recurrent compounding errors, high inference-time latency, and instability in backward propagation through time (BPTT). In this work, we propose a model-based RL approach that uses the chunk world model as a differentiable training-time evaluator to optimize the chunk actor without recurrent model invocation. Through the collaborative chunk-wise prediction of the world model and policy, we can generate H-step hypothetical rollouts in a single forward pass and directly optimize value expansion, avoiding the costly and unstable BPTT. We evaluate our method on 19 challenging high-dimensional humanoid tasks from HumanoidBench and deploy our method to the real-world for quadruped bipedal walking. \href{https://song-wx.github.io/cmpo.github.io/}{Project Website.}
\end{abstract}



\section{Introduction}
Reinforcement learning (RL) has achieved remarkable success in solving complex continuous control problems~\citep{hoeller2024anymal,radosavovic2024real,handa2023dextreme}, games~\citep{ye2021mastering,hafner2020mastering} and large language model post-training~\citep{shao2024deepseekmath,jaech2024openai}. Modern model-free algorithms~\citep{haarnoja2018soft, schulman2017proximal,fujimoto2018addressing,lillicrap2020continuous} provide general-purpose optimization methods and strong asymptotic performance, but they often require a large number of environment interactions. Model-based reinforcement learning (MBRL) offers a promising alternative by learning a surrogate world model that can provide additional information beyond real environment interactions. Modern MBRL methods have shown strong sample efficiency by leveraging learned latent dynamics in different ways. Dyna-style methods leverage the learned model to generate additional transitions as synthetic training data augmentation~\citep{sutton1991dyna,janner2019trust,yu2020mopo,hafner2020mastering,hafner2023mastering,hafner2019dream}, while MPC-style methods~\citep{hansen2022temporal,hansen2024td,okada2020planet,chua2018deep,wang2025bootstrapped} use the learned model for decision-time planning. Another line of MBRL methods~\citep{amos2021model,clavera2020model,amigo2025first}, termed as model-gradient-based methods, exploit the learned dynamics as the differentiable simulator to backpropagate policy gradients through predicted trajectories, providing more informative updates than derivative-free policy search. 

Despite the different optimization forms, these approaches share a common structural dependency: they generate hypothetical future trajectories through recurrent model unrolling, often interleaved with repeated policy queries. This recursive invocation introduces several challenges for efficient policy optimization. First, generating an \(H\)-step hypothetical trajectory requires \(O(H)\) sequential model calls, which increases both training cost and, for planning-based methods, inference-time latency. Second, when policy gradients are propagated through these recurrent rollouts, optimization must rely on backpropagation through time (BPTT), which is costly and unstable over long horizons.

Action chunking, which predicts a sequence of future actions instead of a single-step action, has been widely adopted in imitation learning to improve temporal consistency and reduce compounding action errors~\citep{zhao2023learning,chi2025diffusion,black2024pi_0}. Recently, this idea has also begun to be explored in reinforcement learning. In model-free RL, Q-Chunking~\citep{li2026reinforcement} and Decoupled Q-Chunking~\citep{li2025decoupled} show that operating over action chunks enables unbiased \(n\)-step temporal-difference backups, thereby reducing TD bootstrap error and improving policy learning stability. Beyond model-free settings, action chunks have also been incorporated into offline model-based RL. MAC~\citep{park2025scalable} jointly learns a chunk policy and a long-horizon dynamics model to scale offline model-based policy learning over extended horizons. These developments suggest that action chunking can serve as a structural tool for improving temporal credit assignment and long-horizon policy optimization. However, existing chunk-based RL methods have primarily focused on reducing bootstrap in value learning, while the potential of chunk-wise modeling for model-gradient-based policy optimization remains underexplored. This naturally raises the following question: can we simultaneously apply chunk-wise modeling to both the policy and the world dynamics model to retain the benefits of model-gradient-based optimization while avoiding the pathologies caused by recurrent model calls?

\begin{wrapfigure}{r}{0.40\textwidth}
\begin{minipage}{\linewidth}
    \includegraphics[width=\linewidth]{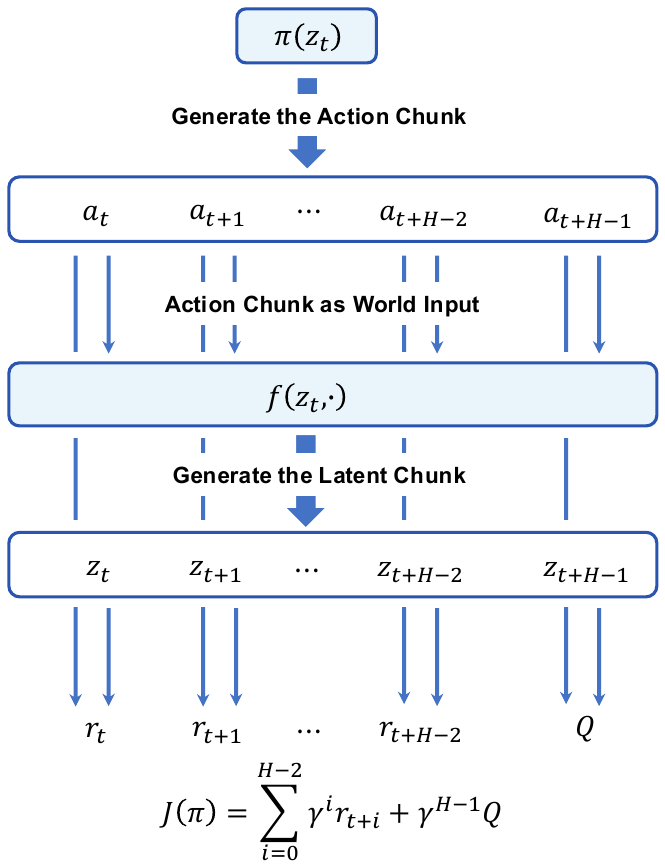}
    \caption{Computation graph of the Collaborative Chunk-wise policy update with the value expansion.}
    \label{fig:compute-graph}
\end{minipage}
\vspace{-1em}
\end{wrapfigure}

To answer this question, we propose \textbf{C}hunk-wise \textbf{M}odel-based \textbf{P}olicy \textbf{O}ptimization(CMPO) that introduces a collaborative chunk-wise modeling design for single-call model-gradient-based policy updates. Instead of learning a conventional one-step latent dynamics model \(\mathbf{z_{t+1}}=f(\mathbf{z_t}, \mathbf{a_t})\) and recurrently applying it over the prediction horizon, we learn a chunk world model that directly maps the current latent state and a chunk of future actions to a sequence of future latent states, i.e., \(\mathbf{z_{t+1:t+H}}=f(\mathbf{z_t}, \mathbf{a_{t:t+H-1}})\). Together with a chunk-wise actor \(\pi_\theta(\mathbf{z_t})=\mathbf{a_{t:t+H-1}}\), our method generates an \(H\)-step hypothetical rollout through a single model invocation, as shown in Figure~\ref{fig:compute-graph}. The chunk world model serves as a differentiable simulator, enabling the actor to directly optimize a multi-step value expansion objective without recurrent model-policy interactions that cause unstable BPTT. At inference time, the actor predicts an action chunk but executes only the first action, preserving closed-loop responsiveness, which outperforms MPC-style methods without test-time planning. We systematically evaluate our framework across 19 challenging, high-dimensional humanoid tasks from HumanoidBench\citep{sferrazza2024humanoidbench}. Furthermore, to demonstrate the transferability of our approach, we apply our method to quadruped bipedal walking in the real world. Our experimental results show that CMPO achieves highly efficient policy improvement while maintaining substantially low training and inference budgets, offering an efficient path forward for generalist robot reinforcement learning. 

Overall, the contributions of our work are summarized as follows:
\begin{itemize}[leftmargin=1em]
\item We propose \textbf{Collaborative Chunk-wise Modeling}, which jointly models chunk-level actions and latent dynamics to generate multi-step hypothetical rollouts through a single model invocation.
\item We introduce \textbf{CMPO}, a model-gradient-based RL framework that avoids recurrent model rollout and long-horizon BPTT, improving both optimization stability and computational efficiency.
\item We demonstrate strong performance and sample efficiency on 19 HumanoidBench tasks, along with successful real-world deployment on quadruped bipedal locomotion.
\end{itemize}

\section{Related Work}
\textbf{MBRL with Imagination as Augmentation.} A classical way of using learned models in MBRL is to generate additional imagined experiences for policy or value learning, which can date back to Dyna~\citep{sutton1991dyna}. Model rollouts are synthesized from real replay states and mixed with real transitions to incorporate the model-free learner~\citep{janner2019trust,kurutach2018model,luo2018algorithmic}. In offline RL, methods such as MOPO~\citep{yu2020mopo}, MOReL~\citep{kidambi2020morel}, and COMBO~\citep{yu2021combo} also use model-generated rollouts to improve policy learning, and additionally introduce uncertainty penalties, pessimistic MDPs, or conservative value regularization to mitigate out-of-distribution model exploitation. These methods treat the model primarily as a data generator and the imagined transitions or trajectories are used as additional training samples. 

\textbf{MBRL with Planning.} Another important line of MBRL uses the learned model for decision-time planning. Early planning-based MBRL methods learn dynamics models and perform trajectory optimization through random shooting, the cross-entropy method (CEM), or related sampling-based optimizers~\citep{nagabandi2018neural, chua2018deep}. More recently, TD-MPC series~\citep{hansen2024td,hansen2022temporal} combines short-horizon latent-space trajectory optimization with a learned terminal value function, allowing local model-based planning while relying on temporal-difference learning for long-term return estimation. Further, \citet{wang2025bootstrapped} bootstraps the policy learning by imitating the MPC expert and \citet{lin2025td} introduces policy regularization to reduce the out-of-distribution queries by the planner and the learned policy prior.

\textbf{MBRL with Model Gradient.}
The third category of MBRL exploits the differentiability of the learned model and leverages model gradient to directly optimize the policy. Stochastic Value Gradient(SVG)~\citep{heess2015learning} provides a general framework for continuous control by backpropagating through stochastic dynamics, value functions, and policies. \citet{amos2021model} combines SVG with Soft Actor Critic~\citep{haarnoja2018soft} to improve the asymptotic performance of model-based methods and Model-Augmented Actor-Critic (MAAC)~\citep{clavera2020model} further derives the monotonic improvement of the pathwise derivative in terms of the gradient error in the model and the value function. Dreamer series~\citep{hafner2019dream,hafner2020mastering,hafner2023mastering} also optimizes the actor by propagating gradients of imagined returns through compact latent trajectories. \citet{amigo2025first} instead performs trajectory unrolling with parallel simulators and propagates gradients through the learned model to mitigate compounding errors.

These categories may overlap in hybrid algorithms and our method primarily falls within the model-gradient category, while remaining orthogonal to test-time planning. In Section~\ref{sec:discussion}, we further explore the potential of integrating our method with MPC.

\section{Chunk-wise Model-based Policy Optimization}
\subsection{Background}
This work focuses on the standard reinforcement learning setting modeled as the infinite-horizon Markov Decision Process (MDP), defined by the tuple $(\mathcal{S}, \mathcal{A}, \mathcal{P}, \mathcal{R}, \gamma)$, where $\mathcal{S}$ denotes the state space, $\mathcal{A}$ denotes the action space, $\mathcal{P}(s_{t+1}|s_t,a_t)$ is the environment transition dynamics, $\mathcal{R}(s_t,a_t)$ is the reward function, and $\gamma \in [0,1)$ is the discount factor. The learning objective is to optimize the actor/policy $\pi_\theta(a_t| s_t)$ that maximizes the expected cumulative discounted rewards $\mathbb{E}_\pi\sum_{t=0}^{\infty}\gamma^tr_t$.

Model-based RL (MBRL) additionally learns a surrogate world model $f_\theta(\mathbf{s_{t+1}}|\mathbf{s_t}, \mathbf{a_t})$ to approximate the environment dynamics and the reward function $R(\mathbf{s_t},\mathbf{a_t})$ to generate immediate reward, which enables agents to simulate future trajectories, termed imagined rollouts, bypassing the need for expensive interactions with a realistic environment. Thus, an $H$-step value expansion with TD truncation can be formulated as:
\begin{equation}
    \hat G_t^{(H)}
    =
    \sum_{k=0}^{H-1} \gamma^k R(\mathbf{z_{t+k}}, \mathbf{a_{t+k}})
    +
    \gamma^H Q(\mathbf{z_{t+H}, \pi(z_{t+H})}).
\end{equation}
which can be utilized for test-time MPC planning~\citep{hansen2024td,hansen2022temporal}, constructing model-augmented TD targets for critic learning~\citep{feinberg2018model}, or policy optimization~\citep{amos2021model}.

\subsection{Motivation}
Typical model-based RL methods~\citep{hafner2019dream,hafner2020mastering,hafner2023mastering,hansen2024td,hansen2022temporal} generate hypothetical rollouts through recurrent interactions between the world model and the actor or planner. Given a latent state \(z_t\), a standard one-step latent dynamics model recursively predicts
$
    \mathbf{z_{t+k+1}} = f(\mathbf{z_{t+k}}, \mathbf{a_{t+k}}),\
    \mathbf{a_{t+k}} = \pi_\theta(\mathbf{z_{t+k}}),\
    \mathbf{r_{t+k}} = R(\mathbf{z_{t+k}, \mathbf{a_{t+k}}}),
$
for \(k=0,\ldots,H-1\). The resulting imagined trajectory $(\mathbf{z_\tau},\mathbf{a_\tau},\mathbf{r_\tau})_{\tau=t}^{t+H}$  is then used either for policy learning, value expansion, or decision-time planning. Although this recurrent rollout design provides useful model-based signals, it also introduces several structural limitations.

First, generating an \(H\)-step hypothetical trajectory requires \(\mathcal{O}(H)\) sequential model invocations. Since \(\mathbf{z_{t+k+1}}\) depends on \(\mathbf{z_{t+k}}\), these calls cannot be fully parallelized over the temporal dimension. This increases the training cost for imagination-based policy learning and further induces inference-time latency for MPC-style methods that repeatedly evaluate candidate action sequences.

Second, recurrent model rollouts suffer from recursive compounding prediction errors. Let \(\mathbf{z_{t+k}}\) denote the ground-truth latent state and \(\mathbf{\hat z_{t+k}}\) the predicted state. If the one-step model has prediction error bounded by \(\epsilon\), and the dynamics model is \(L_f\)-Lipschitz with respect to the latent state, then the rollout error can grow as
\begin{equation}
    \|\mathbf{\hat z_{t+k}} - \mathbf{z_{t+k}}\|
    \leq
    \epsilon \sum_{i=0}^{k-1} L_f^i,
\end{equation}
\begin{proof}
    See Appendix~\ref{ap:proof}.
\end{proof}
which increases with the imagination horizon \(k\), and may grow rapidly when \(L_f>1\). Consequently, the value expansion target \(\hat G_t^{(H)}\) can become biased by errors accumulated over the entire imagined trajectory, especially for long-horizon control.

Third, for gradient-based model-based policy optimization, optimizing the policy through such recurrent hypothetical rollouts requires backward propagation through time (BPTT). The policy gradient involves products of transition Jacobians along the imagined trajectory. Specifically, the first-order gradient of future imagined states to the policy parameters \(\theta\) follows
\begin{equation}
    \frac{d \mathbf{\hat z_{t+k+1}}}{d\theta}
    =
    \frac{\partial f}{\partial z}
    \frac{d \mathbf{\hat z_{t+k}}}{d\theta}
    +
    \frac{\partial f}{\partial a}
    \frac{d \pi_\theta}{dz}
    \frac{d \mathbf{\hat z_{t+k}}}{d\theta}=\Bigg(
    \frac{\partial f}{\partial z}
    +
    \frac{\partial f}{\partial a}
    \frac{d \pi_\theta}{dz}\Bigg)\cdot
    \frac{d \mathbf{\hat z_{t+k}}}{d\theta}.
\end{equation}
Unrolling this recursion yields a chain of Jacobian products across the horizon. Therefore, the gradient magnitude can be dominated by terms of the form
\begin{equation}
\label{eq:instable}
    \prod_{j=i+1}^{k}
    \left(
    \frac{\partial f}{\partial z}
    +
    \frac{\partial f}{\partial a}
    \frac{\partial \pi_\theta}{\partial z}
    \right),
\end{equation}
which may lead to exploding or vanishing gradients as \(k\) increases. In addition, BPTT requires storing the entire recurrent computation graph, further increasing memory and computational costs. We empirically verify this instability in Section~\ref{sec:discussion}, where recurrent model rollouts lead to unstable gradient norms and degraded policy updates.

Action chunking\citep{zhao2023learning}, as a technique widely applied in imitation learning for learning temporally consistent behaviors, has recently begun to be explored in model-free RL as well, where it has also been found to reduce temporal difference bootstrap error by using unbiased n-step backups\citep{li2026reinforcement,li2025decoupled}. Actors on chunked action space have also been applied alongside long-horizon world models for scaling offline model-based RL in \citep{park2025scalable}. These works reveal a broader role of action chunking in reinforcement learning: it can reshape the temporal structure of policy optimization, rather than merely generating consistent low-level controls. Existing chunk-based RL methods mainly exploit this structure from the value-learning perspective. However, the same idea is particularly appealing for model-based gradient policy optimization. The key bottlenecks discussed above all stem from a step-wise rollout mechanism, where the model and the actor must be recurrently invoked to construct an \(H\)-step hypothetical trajectory. 

This observation motivates our central question: can we leverage chunk-wise policy and dynamics modeling to preserve the multi-step training signal of model-based gradient optimization, while eliminating recurrent model invocation, the associated compounding error and BPTT issues?

\subsection{Collaborative Chunk-wise Modeling}
The core motivation is that if both the actor and the world model operate directly at the chunk level, the entire hypothetical trajectory can be generated through a single forward pass. Therefore, we model both the policy and the world model in a chunk-wise manner. Given a current latent state \(\mathbf{z_t}\), the chunked actor predicts a sequence of future actions:
\[
    \mathbf{a_{t:t+H-1}} = \pi(\mathbf{z_t}),
\]
where \(H\) denotes the chunk horizon. Instead of using a one-step latent dynamics model \(\mathbf{z_{t+1}}=f(\mathbf{z_t},\mathbf{a_t)}\) and recursively applying it \(H\) times, we train a chunk world model that directly maps the current latent state and the entire action chunk to a future latent state chunk:
\[
    \mathbf{z_{t+1:t+H}} = f(\mathbf{z_t}, \mathbf{a_{t:t+H-1}}).
\]
This collaborative chunk-wise design changes the role of the learned world model. Rather than serving as a recurrent surrogate environment that must be rolled out step by step, the chunk world model acts as a differentiable training-time evaluator for the chunked actor. Given the predicted latent trajectory \((\mathbf{z_t},\mathbf{a_t},\cdots,\mathbf{a_{t+H-1},\mathbf{z_{t+H}}})\), we can evaluate a \((H-1)\)-step value expansion objective:
\[
    G_t^{(H-1)}
    =
    \sum_{k=0}^{H-2}
    \gamma^k R( \mathbf{z_{t+k}}, \mathbf{a_{t+k}})
    +
    \gamma^{H-1} Q(\mathbf{z_{t+H-1}}, \mathbf{a_{t+H}}),
\]
This formulation reduces the sequential rollout cost from \(\mathcal{O}(H)\) recurrent model calls to a single chunk-level model invocation, avoids recursively compounding error mechanism of imagined latents and value estimation, and for model gradient-based policy optimization, the gradients are propagated through one chunk-level prediction that is significantly shorten by \(H\) times. The overall computational graph of the chunk-wise modeling objective is illustrated in Figure~\ref{fig:compute-graph}.

Based on the above discussion, we introduce the practical design for \textbf{C}hunk-wise \textbf{M}odel-based \textbf{P}olicy \textbf{O}ptimization(\textbf{CMPO}).

\textbf{Components.} We build our method upon TD-MPC2\citep{hansen2024td}, a state-of-the-art model-based architecture design. Therefore, apart from the actor and the world dynamic model, the remaining design stays consistent. Specifically, it consists of the following five components:
\begin{equation}
\begin{aligned}
& \text{Encoder} & \mathbf{z_t} &= h(\mathbf{s_t}) & & \triangleright \text{ Maps states to latent representations} \\
& \text{Latent dynamics} & \mathbf{z_{t+1:t+H}} &= f(\mathbf{ z_t}, \mathbf{ a_{t:t+H-1}}) & & \triangleright \text{ Chunked latent forward dynamics} \\
& \text{Reward} & \mathbf{\hat r_t} &= R(\mathbf{z_t}, \mathbf{a_t}, ) & & \triangleright \text{ Predicts reward } r \text{ of a transition} \\
& \text{Terminal value} & \mathbf{\hat q_t} &= Q(\mathbf{z_t}, \mathbf{a_t}) & & \triangleright \text{ Predicts discounted sum of rewards } \\
& \text{Policy prior} & \mathbf{a_{t:t+H-1}} &= \pi(\mathbf{z_t}) & & \triangleright \text{ Predicts chunked action }
\end{aligned}
\end{equation}

\textbf{Model Learning.} Following the basic framework of the prior work, the training objective of the state encoder $h$, the chunked world dynamics model $f$, and the reward model $R$ is summarized as follows:
\begin{equation}
\label{eq:model_obj}
\begin{gathered}
    \mathcal{L}_{\text{model}} = \underset{(\mathbf{s_t},\mathbf{a_t},\mathbf{r_t}, \mathbf{s_{t+1}})_{0}^{H-1}\sim\mathcal{B}}{\mathbb{E}}\Bigg[\sum_{t=0}^{H} \lambda^t\bigg(||\mathbf{z_t}- \text{sg}(h(\mathbf{s_{t}}))||_2^2 + \text{CE}(\mathbf{\hat{r}_{t}}, \mathbf{r_{t+1}})\bigg) \Bigg] \\
    \mathbf{z_{1:H}} = f(\mathbf{z_0},\mathbf{a_{0:H-1}})
\end{gathered}
\end{equation}
where $\text{sg}$ is the stop-gradient operation, $\lambda$ is a hyper-parameter that weights the near-prediction terms higher, $\text{CE}$ is the cross-entropy loss of the two-hot distribution, and $\mathbf{z_{1:H}}$ is the chunked rollout by the chunked world dynamics $f$.

\textbf{Value Learning.} Similar to \citep{li2025decoupled}, although we adopt a chunked actor as the base policy, during inference we only execute the first step of the action chunk to ensure close-loop responsiveness to the environment. Therefore, different from the multi-step action Q function used in Q-Chunking\citep{li2026reinforcement}, we employ a single-step action Q functions as the critic and update the critic with the classic 1-step TD error:
\begin{equation}
\label{eq:value_obj}
\begin{gathered}
    \mathcal{L}_{\text{critics}} = \underset{(\mathbf{s_t},\mathbf{a_t},\mathbf{r_t}, \mathbf{s_{t+1}})_{0}^{H-1}\sim\mathcal{B}}{\mathbb{E}}\Bigg[\sum_{t=0}^{H} \lambda^t\text{CE}(\hat{q}_{t}, q_t) \Bigg] \\
    q_t = \text{sg}(\mathbf{r_{t+1}}+Q'(\mathbf{z_{t+1}},\pi_\theta(\mathbf{z_{t+1}}))),\ \mathbf{z_{t+1}}=h(\mathbf{s_{t+1}})
\end{gathered}
\end{equation}
where \(q_t\) is the TD target and \(Q'\) is the target Q function updated via exponential moving average.

\textbf{Policy Learning.}
We use model-based value expansion~\citep{feinberg2018model} as our training-time policy optimization objective, which is closely related to the objective in Stochastic Value Gradient methods~\citep{amos2021model} and Model-Augmented Actor-Critic~\citep{clavera2020model}. Since both the actor and the world model operate at the chunk level, the actor can optimize the single chunk-level rollout with ease, avoiding recurrent model unrolling and long-horizon BPTT. Instead of using a full \(H\)-step value expansion objective, which requires querying the actor again at the terminal imagined latent \(\mathbf{z_{t+H}}\), we adopt an \((H-1)\)-step surrogate. With SAC-style entropy regularization~\citep{haarnoja2018soft}, the actor objective is:
\begin{equation}
\label{eq:policy_obj}
\begin{gathered}
    \mathcal{L}_{\text{policy}} = \underset{\substack{\mathbf{s} \sim \mathcal{B} \\ \mathbf{a_{0:H-1}} \sim \pi}}{\mathbb{E}}\left[ -\alpha(\sum_{i=0}^{H-2}\gamma^iR(\mathbf{z_{i}}, \mathbf{a_{i}}) +\gamma^{H-1}Q(\mathbf{z_{H-1}}, \mathbf{a_{H-1}})) - \beta\mathcal{H}(\pi(\cdot|\mathbf{s}))\right]\\
    \mathbf{z_{1:H}} = f(\mathbf{z_0},\mathbf{a_{0:H-1}})
\end{gathered}
\end{equation}
where $\mathcal{H}$ is the entropy of the action distribution, the chunked rollout $\mathbf{z_{1:H}}$ is inferred by the one-step invocation of the chunked world dynamics $f$ and \(\alpha,\beta\) are the hyper-parameters that control the degree of entropy regularization. We provide the pseudocode for CMPO in Appendix~\ref{ap:alg} and the detailed architecture design in Appendix~\ref{ap:net}.

\section{Experiments} 
\begin{figure}[t]
    \centering
    \includegraphics[width=\linewidth]{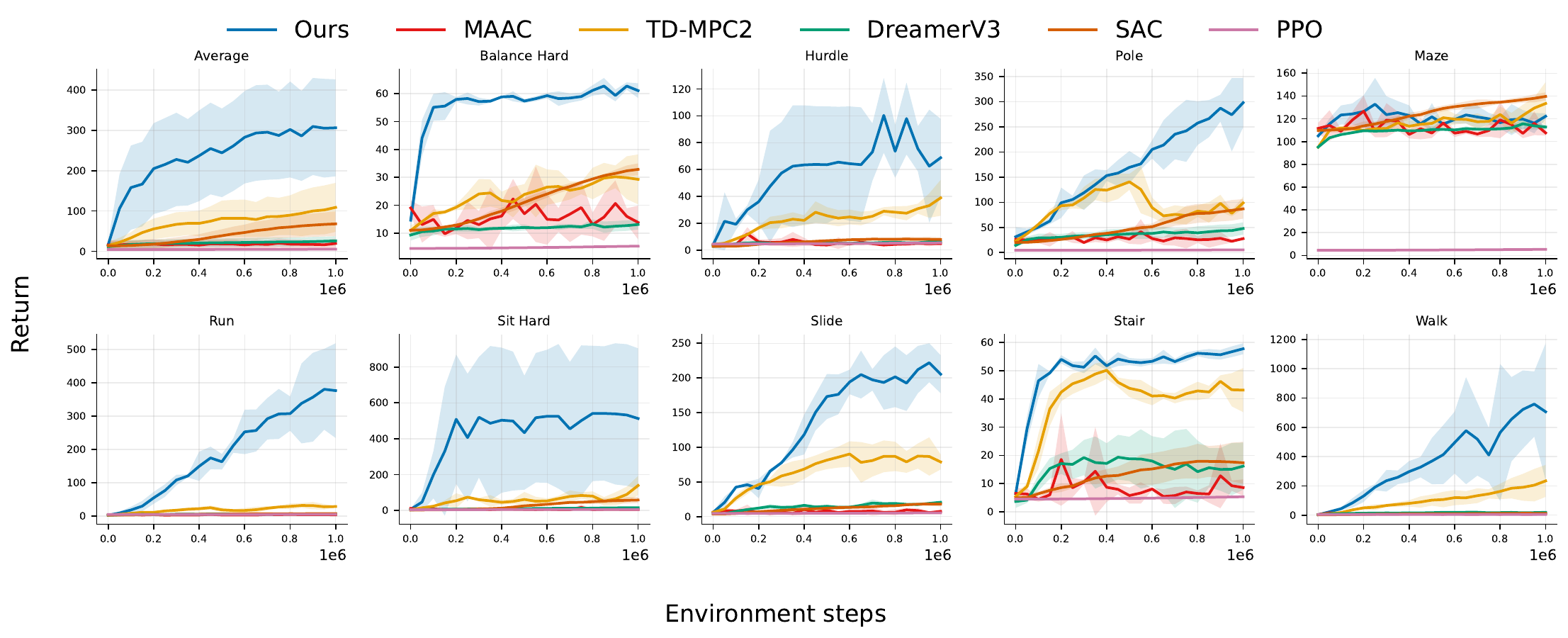}
    \caption{\textbf{Results in HumanoidBench (9 tasks shown). }Mean of 3 runs with different seeds. Shaded areas are 95\% confidence intervals. CMPO achieves the best average performance and sample efficiency across baselines. Refer to Appendix~\ref{ap:all_results} for additional results.}
    \label{fig:main-results}
    \vspace{-1em}
\end{figure}
The goal of our empirical evaluation is to assess whether collaborative chunk-wise modeling can benefit efficient model-gradient-based policy optimization. In particular, we investigate whether CMPO can achieve competitive performance and sample efficiency compared with strong model-free and model-based RL baselines, while also evaluating its applicability to real-world robotic systems. Finally, we further analyze several key properties of CMPO, including optimization stability during training, the effects of different horizon settings, and its integration with test-time planning.
\subsection{High-dimensional Humanoid Continuous Control}
We evaluate our method on 19 selected challenging high-dimensional ($\mathcal{A}\in\mathbb{R}^{61}$) continuous control tasks from HumanoidBench~\citep{sferrazza2024humanoidbench}. These tasks require learning coordinated whole-body behaviors under large state and action spaces, which are suitable for evaluating sample efficiency, training stability, and long-horizon policy optimization of RL algorithms. 

\textbf{Baselines.}
We compare our method against both model-free and model-based RL algorithms. For model-free baselines, we include PPO~\citep{schulman2017proximal} and SAC~\citep{haarnoja2018soft}, two widely used continuous-control algorithms with strong empirical performance. For model-based baselines, we compare with TD-MPC2~\citep{hansen2024td}, a strong MPC-style model-based RL method that combines latent dynamics learning with decision-time trajectory optimization, and DreamerV3~\citep{hafner2023mastering}, a representative imagination-based method that trains actors and critics through recurrent latent rollouts and backpropagate the expected returns through the model. We additionally include MAAC~\citep{clavera2020model} as a gradient-based model-based RL baseline. 
This comparison allows us to evaluate whether our chunk-wise formulation can retain the benefits of model-based methods while avoiding the instable optimization of BPPT. We provide the more experimental setting details in the Appendix~\ref{ap:exp}.

\textbf{Results.} We present the averaged results across all 19 selected tasks in Figure~\ref{fig:main-results} and demonstrate a subset of the results. The complete results are provided in Appendix~\ref{ap:all_results}. The results show that CMPO achieves substantially better average performance and surpasses the baselines in the majority of the evaluated tasks. Additionally, CMPO demonstrates a significant advantage in sample efficiency. Furthermore, we compare the training wall-time of TD-MPC2 and CMPO. We benchmark the training wall-time of both methods on a single NVIDIA RTX 4090 under identical settings and without employing optimization techniques such as Torch Compile. The results is as shown in Table~\ref{tab:wall-time}. Benefiting from one-step hypothetical rollout generation enabled by chunk-wise modeling, as well as eliminating iterative zeroth-order optimization during inference, CMPO requires only around 50\% of the training time of TD-MPC2 under the same number of training steps while achieving approximately a 30$\times$ speedup in inference, further demonstrating its practical efficiency.

\begin{wraptable}{r}{0.5\textwidth}
\vspace{-1em}
\centering
\small
\setlength{\tabcolsep}{3pt}
\resizebox{\linewidth}{!}{
\begin{tabular}{ccccc}
\toprule
& \multicolumn{2}{c}{Walk} & \multicolumn{2}{c}{Run} \\
\cmidrule(lr){2-3} \cmidrule(lr){4-5}
Wall-time & TD-MPC2 & CMPO & TD-MPC2 & CMPO \\
\midrule
200k training steps (h) $\downarrow$ &3.57 &\textbf{1.74} &3.54 &\textbf{1.72}\\
one inference step (ms) $\downarrow$& 40.8& \textbf{1.4}& 39.5&\textbf{1.3}\\
\bottomrule
\end{tabular}
}
\caption{\textbf{Training and inference Wall-time.}}
\label{tab:wall-time}
\vspace{-1em}
\end{wraptable}

\subsection{Quadruped Locomotion \& Sim-to-Real Deployment}
\begin{figure}[h]
    \centering
    \includegraphics[width=0.95\linewidth]{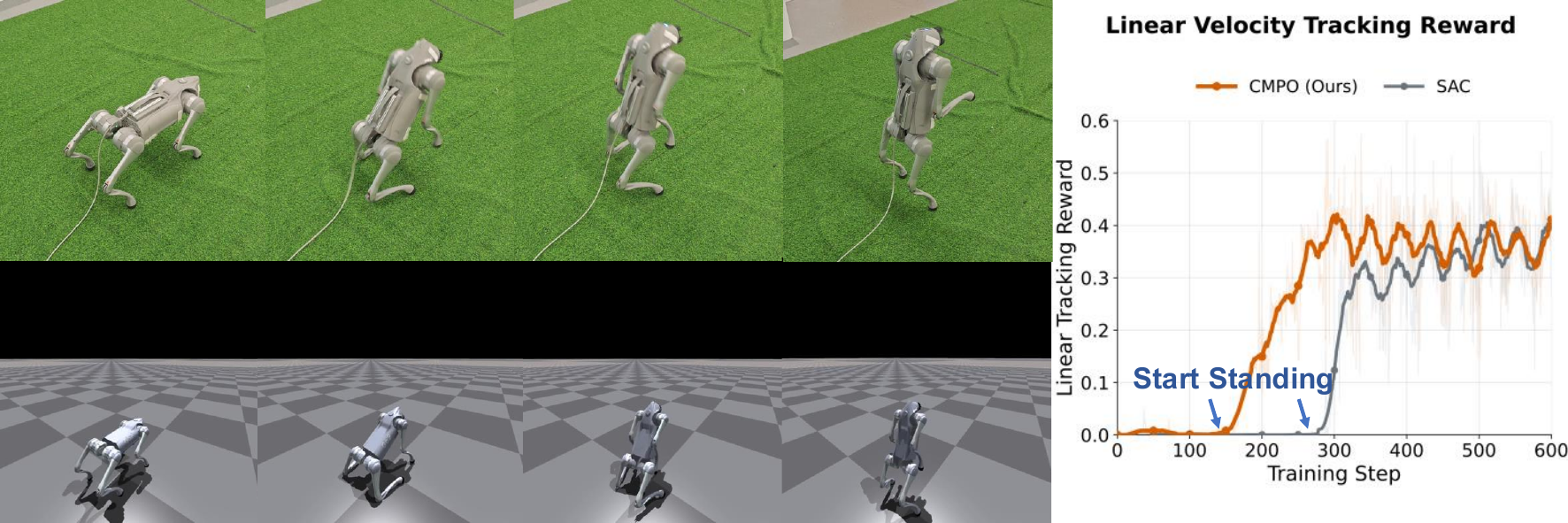}
    \caption{\textbf{Sim-to-real transfer of standing locomotion on Unitree Go2.}: \textbf{Left:} Representative snapshots of the learned standing locomotion policy in simulation and on the real robot. \textbf{Right:} Training curve of the standing linear-tracking reward in simulation. 
}
    \label{fig:sim-2-real}
\end{figure}
One of the key advantages of CMPO is its ability to achieve strong performance with a direct-output inference pipeline which outperforms methods that rely on MPC-based test-time planning. This significantly reduces inference latency as shown in Table~\ref{tab:wall-time}, enabling real-time responses to the dynamic environment, which is both beneficial to training efficiency and the real-world robotic deployment. Therefore, we further conduct sim-to-real deployment experiments on the Unitree Go2 quadruped robot platform. We evaluate SAC and CMPO under the same training and deployment pipeline. Figure~\ref{fig:sim-2-real} shows the sim-2-real evaluation of standing locomotion on Unitree Go2 and the standing linear velocity tracking results of SAC and CMPO. The reward is defined as the linear velocity tracking objective gated by the robot base height, such that the tracking reward is activated only when the base height exceeds a predefined threshold. In simulation, CMPO discovers standing substantially earlier and achieves higher tracking reward. On hardware, maintaining the bipedal posture for 20s succeeds in \textbf{7/10} trials for CMPO versus \textbf{4/10} for SAC. We use SAC because TD-MPC2’s iterative test-time planning is not directly compatible with the same low-latency TorchScript feed-forward deployment interface. Full reward are provided in the Appendix ~\ref{ap:real}.

\subsection{Discussion}
\label{sec:discussion}

\begin{figure}[h]
    \centering
    \begin{subfigure}{0.5\textwidth}
        \centering
        \includegraphics[width=\linewidth]{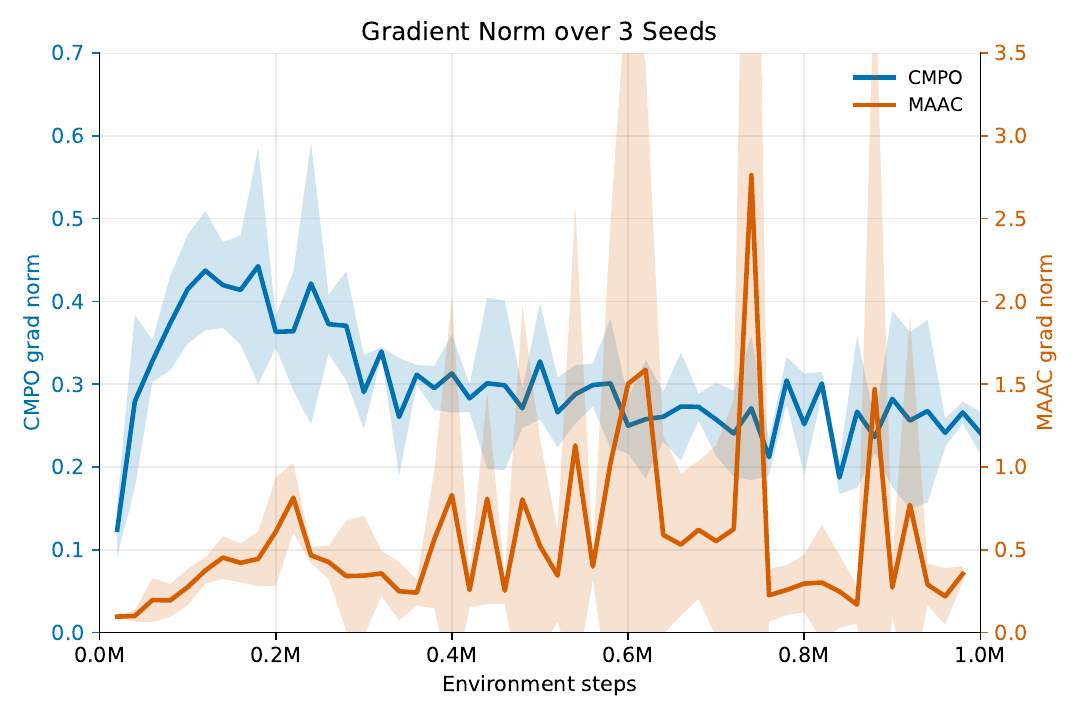}
        \caption{\textbf{Gradient instability in MAAC.}}
        \label{fig:instable}
    \end{subfigure}
    \hfill
    \begin{subfigure}{0.35\textwidth}
        \centering
        \includegraphics[width=\linewidth]{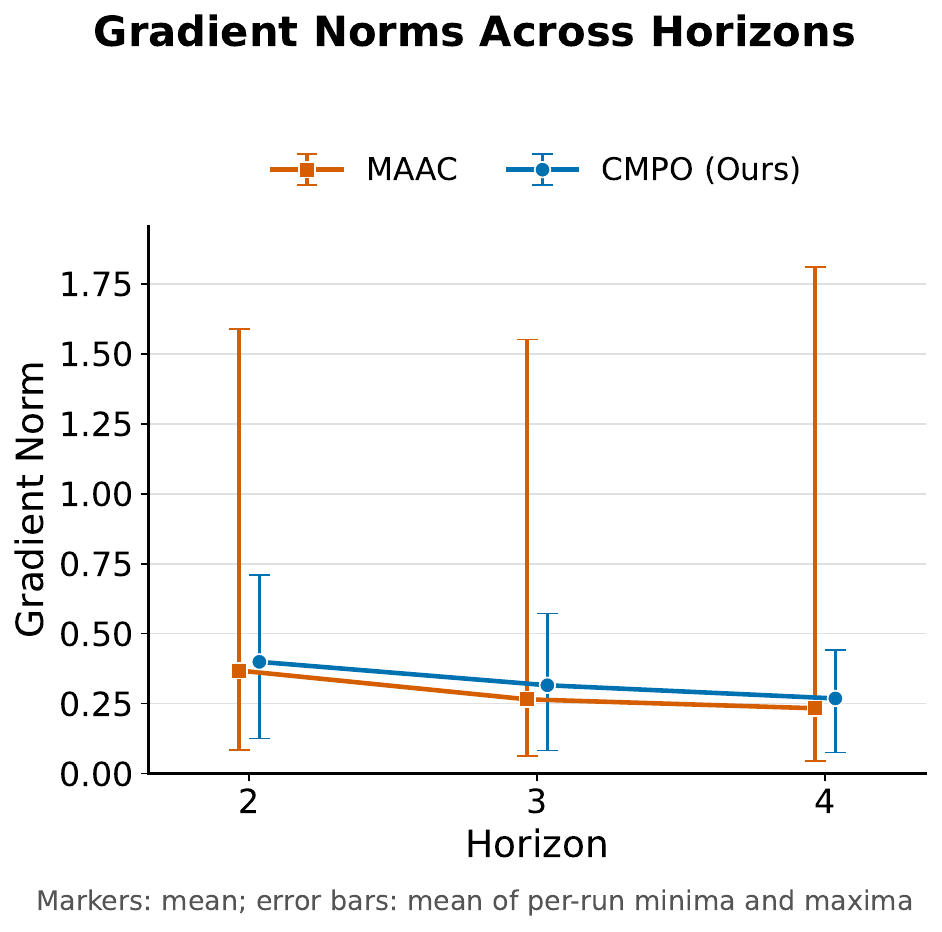}
        \caption{\textbf{Gradient Norm across Horizons.}}
        \label{fig:grad_norm_vers_horizon}
    \end{subfigure}
\end{figure}

\textbf{Gradient instability in MAAC.} One of our primary motivations is to employ chunk-wise modeling for both the policy and the world model to mitigate the optimization instability introduced by BPTT over long horizons. To validate this, we report the gradient norms during training for the classical single-step modeling approach MAAC and compare them with those of CMPO. As shown in Figure~\ref{fig:instable}, we find that MAAC exhibits highly oscillatory gradient norms, consistent with the characteristics of compositional gradient operators~\ref{eq:instable}, while CMPO, benefiting from chunk-wise modeling, demonstrates substantially more stable optimization dynamics. We also highlight BPTT pathology is caused by recurrence and can already manifest at moderate horizons. As shown in Fig.~\ref{fig:grad_norm_vers_horizon}, MAAC exhibits larger gradient excursions, whereas CMPO maintains a narrower and more stable range with horizon increasing. 

\begin{figure}[h]
    \centering
    \begin{subfigure}{0.5\textwidth}
        \centering
        \includegraphics[width=\linewidth]{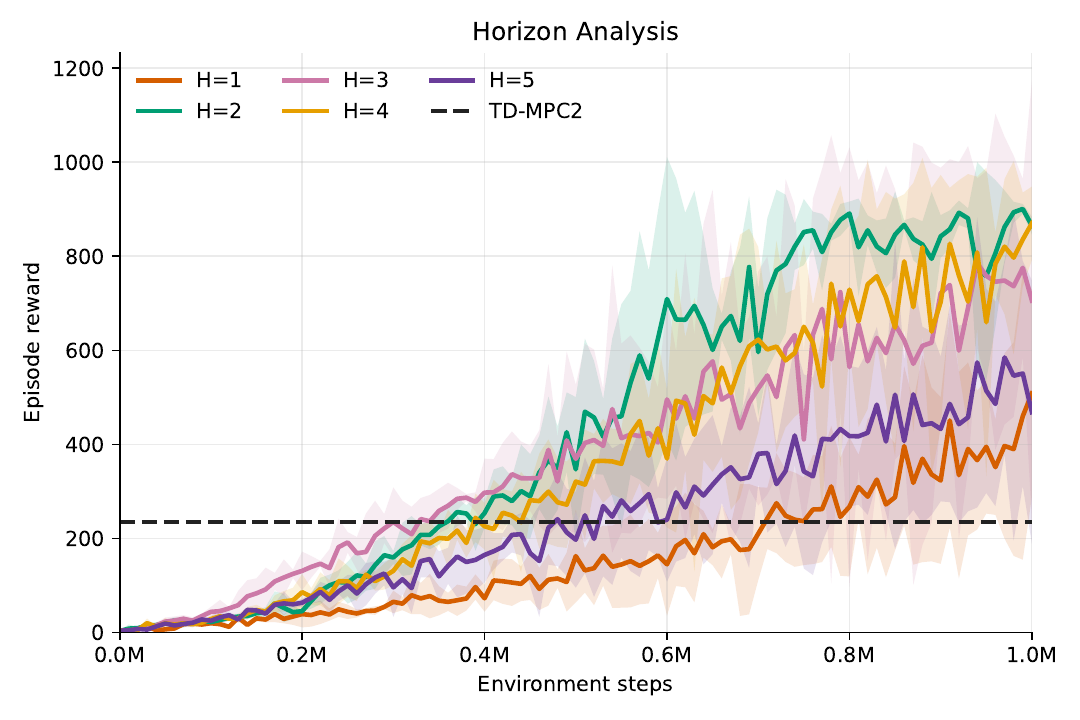}
        \caption{\textbf{CMPO across Horizons.}}
        \label{fig:horizon}
    \end{subfigure}
    \hfill
    \begin{subfigure}{0.35\textwidth}
        \centering
        \includegraphics[width=\linewidth]{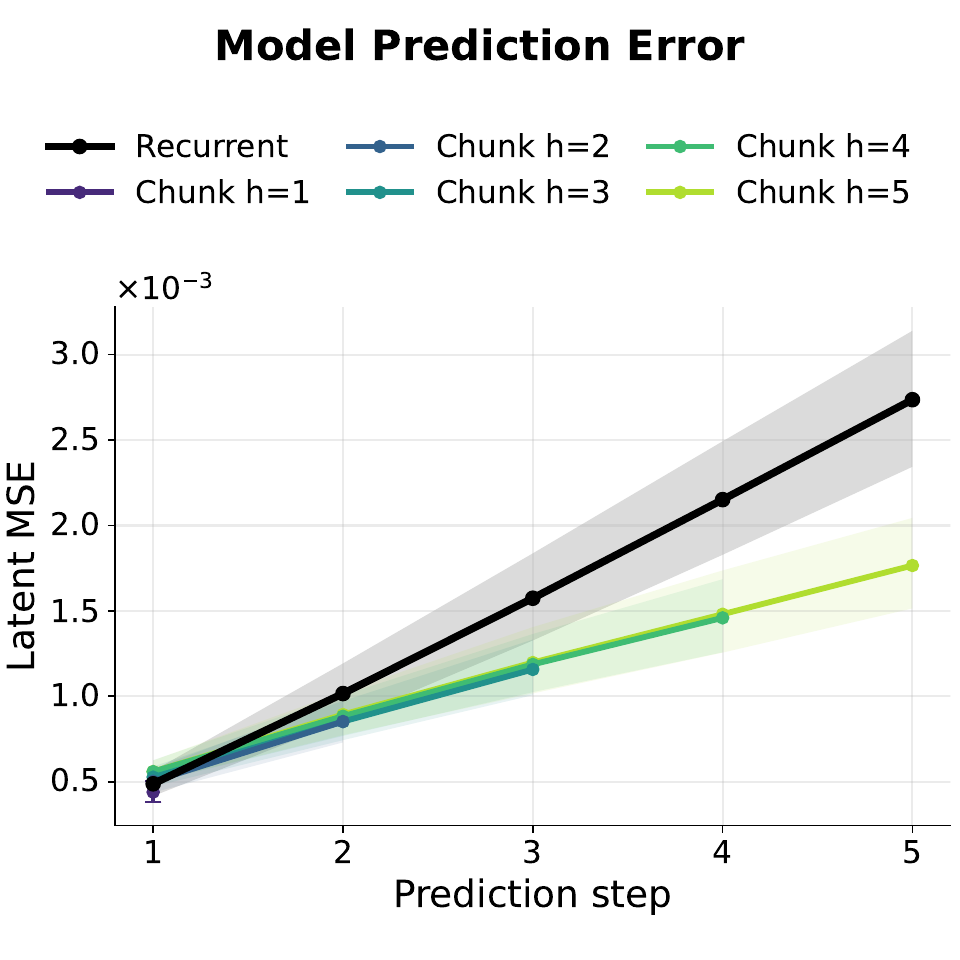}
        \caption{\textbf{Model Prediction Error.}}
        \label{fig:horizon-err}
    \end{subfigure}
\end{figure}
\textbf{CMPO with diverse horizon.} Collaborative chunk-wise modeling avoids recursive model prediction error propagation. Therefore, as shown in Fig.~\ref{fig:horizon-err}, the recurrent rollout error of single step model grows substantially faster with prediction step increasing, while chunk models consistently achieve lower latent prediction error. Also, the chunk-model error increases with horizon, explaining the degradation at larger horizons shown in Figure~\ref{fig:horizon}. We further investigate the performance of CMPO under different horizon settings. The results in Figure~\ref{fig:horizon} indicate that chunk-wise modeling with an appropriate horizon can substantially improve performance, with the best performance achieved at horizon = 2 and 4. In the main experiments, for a fair comparison with TD-MPC2, we consistently use horizon = 3 across all settings, which further demonstrates the strong performance of CMPO even without optimal horizon tuning. We observe that performance degrades when the horizon increases to 5, suggesting that prediction errors remain a limiting factor under longer-horizon rollouts. Overall, CMPO removes the recurrent temporal depth of model rollout/BPTT, but does not eliminate the intrinsic difficulty of direct long-horizon prediction, explaining why a moderate horzion is optimal.

\begin{wraptable}{r}{0.4\textwidth}
\centering
\small
\setlength{\tabcolsep}{3pt}
\resizebox{\linewidth}{!}{
\begin{tabular}{l|c|c|c|c}
\toprule
& Walk & Run & Stair & Slide\\
\midrule
CMPO & \textbf{954.1}& 629.2 & 60.9 & 265.8\\
\midrule
w MPC & 931.1 & \textbf{864.5} & \textbf{71.9} & \textbf{391.2}\\
w/o CM&22.2 &9.6 &37.1 &36.1\\
\bottomrule
\end{tabular}
}
\caption{Best average performance of variants with model prediction control (w MPC) and without the chunk-wise model (w/o CM).} 
\label{tab:w-mpc}
\vspace{-1em}
\end{wraptable}

\textbf{CMPO without the chunk-wise model.} Our key design lies in the collaborative chunk-wise modeling of both the policy and world dynamics model to enable one-step hypothetical rollout generation. A natural question is whether chunk-wise modeling remains effective when applied to only one component. We note that combining a chunk-wise model with a single-step policy is ill-posed, as future latent-state chunks cannot be properly grounded by single-step actions. Therefore, in Table~\ref{tab:w-mpc}, we only evaluate the setting with a chunk-wise policy and a single-step model (w/o CM). The results indicate that while the chunk-wise policy eliminates repeated policy calls, the recursive invocations required by the single-step model still lead to significant performance degradation, further highlighting the negative effects of pathological BPTT dynamics in model-gradient-based methods.

\textbf{CMPO with test-time model-based planning.}
To highlight the effectiveness of chunk-wise modeling itself, we do not adopt MPC-based test-time planning in the main experiments. Since chunk-wise modeling and MPC are orthogonal design choices, we further combine CMPO with the same MPPI planner used in TD-MPC2 and report the results in Table~\ref{tab:w-mpc}. We observe that MPC improves performance on unsaturated tasks, likely due to the additional exploration induced by MPPI's sampling-based zeroth-order optimization. Meanwhile, on simpler tasks, the gradient information provided by chunk-wise modeling alone is sufficient to reach saturated performance.
\section{Conclusion}
We propose a collaborative chunk-wise modeling framework for Model-Based Reinforcement Learning, in which both the actor and the world dynamics model generate chunk-level foresight predictions. This chunk-wise design enables the generation of hypothetical future rollouts without recurrent model invocation, thereby reducing the computational overhead introduced by recursive calls. Moreover, in model-gradient-based training, it effectively stabilizes the optimization under long horizons by mitigating gradient oscillations and avoiding the pathological BPTT dynamics associated with single-step models. The simulation results on high-dimensional humanoid continuous control tasks demonstrate that our method, relying solely on model-gradient-based optimization, surpasses both model-based methods that employ test-time planning and strong model-free baselines. Finally, the real-world demonstrations of bipedal walking on the Unitree Go2 quadruped robot further validate its practical efficiency. 
\section*{Limitation}
In this work, we primarily focus on building a stable model-gradient-based method on top of TD-MPC2, and most of the model architecture is inherited from the design of TD-MPC2. Exploring network architectures that are more suitable for chunk-wise modeling in reinforcement learning remains an open question. In addition, to ensure closed-loop responsiveness in dynamic environments, our chunk policy executes only the first action of each predicted chunk. However, for less dynamic environments, this design implies that the remaining actions within the chunk are discarded and therefore underutilized. While some prior work~\citep{liang2026adaptive,wen2026adaptive,zhang2026chain} in imitation learning has explored dynamic action horizons, this direction remains relatively underexplored in reinforcement learning. Therefore, an important future direction is how to effectively integrate our collaborative chunk-wise modeling framework with dynamic horizon execution.


\clearpage
\acknowledgments{This work was supported in part by Shanghai Municipal Science and Technology Major Project (2025SHZDZX025G14), National Natural Science Foundation of China (U25A20442, 62306175).}


\bibliography{example}  
\appendix

\section{Proof}
\label{ap:proof}
\begin{proposition}[Compounding error of recurrent model rollouts]
\label{thm:compounding_error}
Consider a deterministic latent dynamics function \(f:\mathcal{Z}\times\mathcal{A}\rightarrow\mathcal{Z}\) and a learned dynamics model \(\hat f_\phi:\mathcal{Z}\times\mathcal{A}\rightarrow\mathcal{Z}\). 
Given the same initial latent state \(z_t=\hat z_t\) and the same open-loop action sequence \(a_{t:t+H-1}\), define the true and imagined rollouts as
\(
    z_{t+k+1} = f(z_{t+k}, a_{t+k}),\
    \hat z_{t+k+1} = \hat f_\phi(\hat z_{t+k}, a_{t+k}),\
\)
for \(k=0,\ldots,H-1\). 
Assume that the learned model has uniformly bounded one-step prediction error
\[
    \|\hat f_\phi(z,a)-f(z,a)\| \leq \epsilon,
    \qquad \forall (z,a) \in \mathcal{D},
\]
where \(\mathcal{D}\) contains the states and actions visited by the rollouts. 
Assume further that the true dynamics \(f\) is \(L_f\)-Lipschitz with respect to the latent state, i.e.,
\[
    \|f(z,a)-f(z',a)\| \leq L_f \|z-z'\|,
    \qquad \forall z,z'\in\mathcal{Z}, \ \forall a\in\mathcal{A}.
\]
Then, for any rollout step \(k\geq 1\), the imagined rollout error satisfies
\[
    \|\hat z_{t+k}-z_{t+k}\|
    \leq
    \epsilon \sum_{i=0}^{k-1} L_f^i.
\]
\end{proposition}

\begin{proof}
Let
\[
    e_k = \|\hat z_{t+k}-z_{t+k}\|.
\]
Since both rollouts start from the same initial latent state, we have \(e_0=0\). For the next-step error,
\[
\begin{aligned}
e_{k+1}
&= \|\hat z_{t+k+1}-z_{t+k+1}\| \\
&= \|\hat f_\phi(\hat z_{t+k},a_{t+k}) - f(z_{t+k},a_{t+k})\| .
\end{aligned}
\]
By adding and subtracting \(f(\hat z_{t+k},a_{t+k})\), we obtain
\[
\begin{aligned}
e_{k+1}
&= \|\hat f_\phi(\hat z_{t+k},a_{t+k})
      - f(\hat z_{t+k},a_{t+k}) \\
&\qquad + f(\hat z_{t+k},a_{t+k})
      - f(z_{t+k},a_{t+k})\| \\
&\leq
\|\hat f_\phi(\hat z_{t+k},a_{t+k})
      - f(\hat z_{t+k},a_{t+k})\| \\
&\qquad +
\|f(\hat z_{t+k},a_{t+k})
      - f(z_{t+k},a_{t+k})\| .
\end{aligned}
\]
Using the bounded one-step model error and the Lipschitz continuity of \(f\), we have
\[
    e_{k+1} \leq \epsilon + L_f e_k .
\]
Unrolling this recursion gives
\[
\begin{aligned}
e_k 
&\leq \epsilon + L_f \epsilon + L_f^2 \epsilon + \cdots + L_f^{k-1}\epsilon \\
&= \epsilon \sum_{i=0}^{k-1} L_f^i .
\end{aligned}
\]
Equivalently, when \(L_f\neq 1\),
\(
    \|\hat z_{t+k}-z_{t+k}\|
    \leq
    \epsilon \frac{L_f^k-1}{L_f-1},
\)
and when \(L_f=1\),
\(
    \|\hat z_{t+k}-z_{t+k}\|
    \leq
    k\epsilon.
\)
This completes the proof.
\end{proof}

\section{Implementation Details}
\subsection{Pseudocode for CMPO}
\label{ap:alg}
\begin{algorithm}[H]
\caption{Chunk-wise Model-based Policy Optimization (CMPO)}
\label{alg:wcg_training}

\KwRequire{
Initialize chunk policy $\pi$, chunk world model $f$, critics $Q_\psi$, targte critics $Q_{\psi'}$,state encoder $h$ and reward model $R$ randomly.
}

\While{not converged}{
    \Comment{Collect experiences}
    \For{step $t = 0 \ldots T$}{
        Generate chunked action distribution $\pi_{t:t+C} \leftarrow \pi(\cdot \mid h(\mathbf{s}_t), p_\theta)$\;
        Sample chunked action sequence $\mathbf{a}_{t:t+C} \sim \pi_{t:t+C}$\;
        Execute the first action $(\mathbf{s}_{t+1}, r_t) \leftarrow \mathrm{env.step}(\mathbf{a}_t)$\;
        Store the trainsition $\mathcal{B} \leftarrow \mathcal{B} \cup
        (\mathbf{s}_t, \mathbf{a}_t, r_t, \mathbf{s}_{t+1})$\;
    }
    \Comment{Update networks}
    \For{num updates per episode}{
        Sample chunked batch $\{\mathbf{s}_t, \mathbf{a}_t, r_t, \mathbf{s}_{t+1}, \pi_t\}_{t:t+H}
        \sim \mathcal{B}$\;

        Update $h, f, R$ via Equation \ref{eq:model_obj}\;
        Update $Q_\psi$ via Equation \ref{eq:value_obj}\;
        Update $\pi_\theta$ via Equation \ref{eq:policy_obj}\;
        Update $Q_{\phi^-}$ via EMA.\\
    }
}
\end{algorithm}
\subsection{Network Settings}
\label{ap:net}
We retain most of the network architecture and hyperparameter settings from TD-MPC2 to ensure that the performance gains of CMPO arise solely from our collaborative chunk-wise modeling. Our chunk-wise policy and dynamics model are built upon the MLP-based single-step policy/model architecture of TD-MPC2. Concretely, we extend the output size of the single-step policy from the original $action\ dim$ to $action\ dim \times horizon$. For the chunk-wise dynamics model, the input dimension is expanded from $latent\ dim+action\ dim$ to $latent\ dim+horizon\times action\ dim$, and the output dimension is correspondingly extended from $latent\ dim$ to $latent\ dim\times horizon$. Such a design introduces only minimal modifications to the original TD-MPC2 architecture, ensuring that the observed performance gains primarily stem from collaborative chunk-wise modeling rather than architectural differences.

\section{Experimental Details}
\label{ap:exp}
\subsection{Simulation Benchmark Overview}
We evaluate our method in 19 selected high-dimensional($\mathcal{A}\in\mathbb{R}^{61}$) whole-body tasks with a Unitree H1hand robot which are originally from HumanoidBench~\citep{sferrazza2024humanoidbench}. These 19 selected tasks cover a diverse set of high-dimensional humanoid control challenges, including basic locomotion such as \emph{walking}, \emph{running}, and \emph{standing}; terrain traversal tasks such as \emph{stairs}, \emph{slides}, \emph{hurdles}, \emph{mazes}, and \emph{poles}; and posture control tasks like \emph{balancing}; as well as object- and scene-interaction tasks such as \emph{door opening}, \emph{cube manipulation}, \emph{bookshelf interaction}, \emph{basketball}, \emph{window manipulation}, and \emph{spoon handling}. Together, these tasks evaluate both whole-body locomotion and fine-grained manipulation capabilities, providing a comprehensive benchmark for assessing the robustness and generality of our method.Detailed reward specifications can be found in the original HumanoidBench paper. We additionally provide visualizations for the selected tasks.

\begin{figure}[h]
\centering

\begin{subfigure}{0.18\textwidth}
\includegraphics[width=\linewidth]{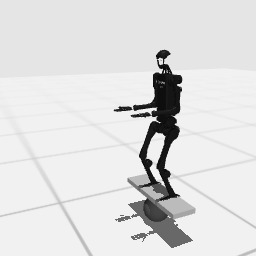}
\caption{Balance}
\end{subfigure}
\begin{subfigure}{0.18\textwidth}
\includegraphics[width=\linewidth]{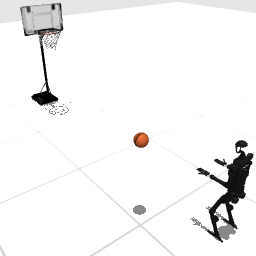}
\caption{Basketball}
\end{subfigure}
\begin{subfigure}{0.18\textwidth}
\includegraphics[width=\linewidth]{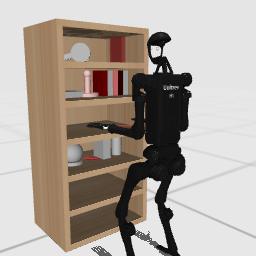}
\caption{Bookshelf}
\end{subfigure}
\begin{subfigure}{0.18\textwidth}
\includegraphics[width=\linewidth]{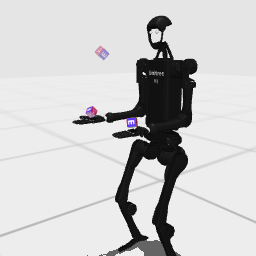}
\caption{Cube}
\end{subfigure}
\begin{subfigure}{0.18\textwidth}
\includegraphics[width=\linewidth]{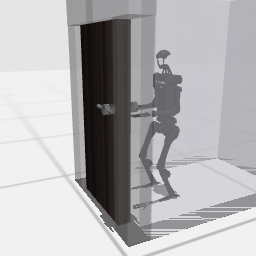}
\caption{Door}
\end{subfigure}
\begin{subfigure}{0.18\textwidth}
\includegraphics[width=\linewidth]{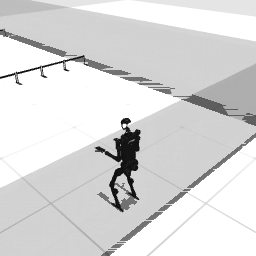}
\caption{Hurdle}
\end{subfigure}
\begin{subfigure}{0.18\textwidth}
\includegraphics[width=\linewidth]{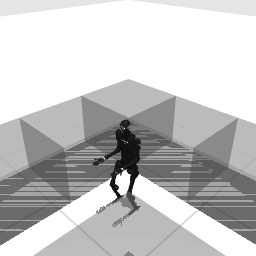}
\caption{Maze}
\end{subfigure}
\begin{subfigure}{0.18\textwidth}
\includegraphics[width=\linewidth]{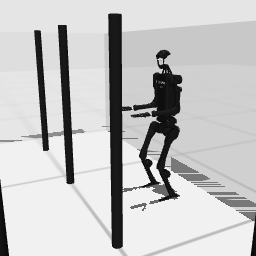}
\caption{Pole}
\end{subfigure}
\begin{subfigure}{0.18\textwidth}
\includegraphics[width=\linewidth]{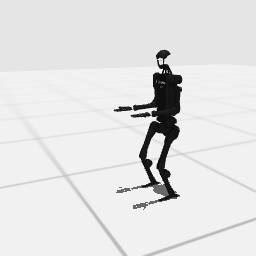}
\caption{Run/Stand/Walk}
\end{subfigure}
\begin{subfigure}{0.18\textwidth}
\includegraphics[width=\linewidth]{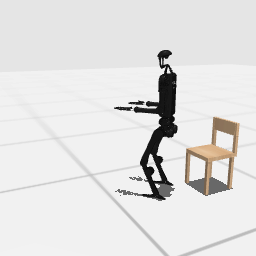}
\caption{Sit }
\end{subfigure}
\begin{subfigure}{0.18\textwidth}
\includegraphics[width=\linewidth]{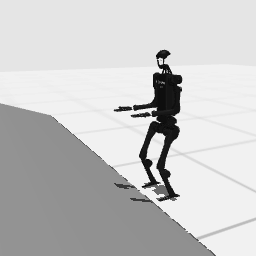}
\caption{Slide}
\end{subfigure}
\begin{subfigure}{0.18\textwidth}
\includegraphics[width=\linewidth]{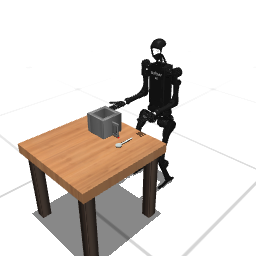}
\caption{Spoon}
\end{subfigure}
\begin{subfigure}{0.18\textwidth}
\includegraphics[width=\linewidth]{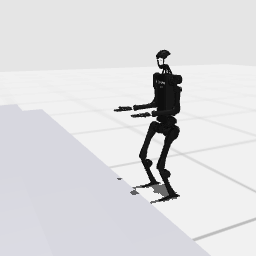}
\caption{Stair}
\end{subfigure}
\begin{subfigure}{0.18\textwidth}
\includegraphics[width=\linewidth]{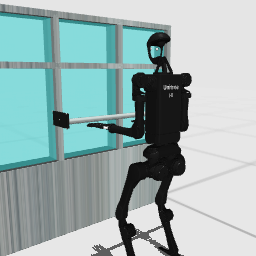}
\caption{Window}
\end{subfigure}
\caption{Visualization of the selected tasks.}
\label{fig:task_visualization}
\end{figure}
\subsection{Simulation Baseline Specification}
For baseline comparisons, we directly adopt the reported results of SAC, PPO, DreamerV3, and TD-MPC2 from the original HumanoidBench benchmark to ensure consistency with the official evaluation protocol. In addition, we implement MAAC-style policy updates on top of the TD-MPC2 architecture as an additional baseline, enabling a more controlled comparison between classical single-step model-gradient optimization and our collaborative chunk-wise modeling framework.

\subsection{Real World Experiment Setup}\label{ap:real}

To evaluate the real-world performance and sim-to-real transferability of our method, we conduct experiments on a bipedal locomotion task using a Unitree Go2 quadruped. We first train the locomotion policy in Isaac Gym and subsequently deploy the learned policy on the real robot. The policy controls all 12 actuated joints through residual joint-position commands, which are tracked by a PD controller. The reward function encourages the robot to maintain an upright bipedal posture while tracking commanded linear and angular velocities, together with regularization terms on joint acceleration, power consumption, action variation, collisions, and hip motion. During simulation training, the robot learns to stand and locomote using its hind legs, and the resulting policy is transferred to the physical Go2 platform for real-world evaluation.
\subsubsection{Reward Design}
Here, $\mathbf{v}^{b}$ and $\boldsymbol{\omega}^{b}$ denote the
base linear and angular velocities expressed in the robot body frame,
respectively; $\mathbf{g}^{\mathrm{proj}}$ is the gravity vector
projected into the body frame; $\mathbf{q}$, $\dot{\mathbf{q}}$, and
$\boldsymbol{\tau}$ denote joint position, velocity, and torque;
and $\mathbf{a}_t$ denotes the policy action.
$\mathcal{C}$ contains the thigh, calf, hip, and base links, while
$\mathcal{H}$ denotes the four hip joints.
We use $\sigma=0.25$ for velocity tracking.
The upright reference configuration is
$q_{\mathrm{hip}}^{\mathrm{ref}}=0$,
$q_{\mathrm{thigh}}^{\mathrm{ref}}=1.8$, and
$q_{\mathrm{calf}}^{\mathrm{ref}}=-1.1$ rad for each leg. The total reward is computed as $r_t = \sum_i w_i r_i$. Here, $\mathbf{1}_{\mathrm{up}} = \mathbf{1}[-g_x^{\mathrm{proj}} > 0.9]$ indicates that the robot is sufficiently upright. The listed weights correspond to the configured reward scales.
\begin{table*}[h]
    \centering
    \caption{\textbf{Reward terms for bipedal Go2 locomotion.}}
    \label{tab:go2_reward}
    \small
    \setlength{\tabcolsep}{5pt}
    \renewcommand{\arraystretch}{1.2}
    \begin{tabularx}{\textwidth}{@{}l X c@{}}
        \toprule
        \textbf{Reward term} & \textbf{Definition} & \textbf{Weight $w_i$} \\
        \midrule

        Linear velocity tracking
        &
        $\displaystyle
        \exp\!\left(
        -\frac{
        (c_x + v_z^{b})^2 +
        (c_y - v_y^{b})^2
        }{\sigma}
        \right)
        \mathbf{1}_{\mathrm{up}}
        $
        & $2.0$
        \\

        Angular velocity tracking
        &
        $\displaystyle
        \exp\!\left(
        -\frac{
        (c_{\mathrm{yaw}}-\omega_x^{b})^2
        }{\sigma}
        \right)
        \mathbf{1}_{\mathrm{up}}
        $
        & $2.0$
        \\

        DOF acceleration
        &
        $\displaystyle
        \sum_j
        \left(
        \frac{\dot q_{t-1,j}-\dot q_{t,j}}
        {\Delta t}
        \right)^2
        $
        & $-2.5\times10^{-7}$
        \\

        Joint power
        &
        $\displaystyle
        \sum_j
        |\dot q_j|\,|\tau_j|
        $
        & $-2.0\times10^{-5}$
        \\

        Base height
        &
        $\displaystyle
        \exp\!\left(
        -\frac{(h-h^\star)^2}{0.25}
        \right),
        \qquad h^\star=0.55\,\mathrm{m}
        $
        & $8.0$
        \\

        Collision
        &
        $\displaystyle
        \sum_{k\in\mathcal{C}}
        \mathbf{1}\!
        \left[
        \|\mathbf f_k\|_2 > 0.1
        \right]
        $
        & $-2.0$
        \\

        Action rate
        &
        $\displaystyle
        \|\mathbf a_t-\mathbf a_{t-1}\|_2^2
        $
        & $-0.01$
        \\

        Action smoothness
        &
        $\displaystyle
        \|\mathbf a_t
        -2\mathbf a_{t-1}
        +\mathbf a_{t-2}\|_2^2
        $
        & $-0.01$
        \\

        Upright orientation
        &
        $\displaystyle
        -g_x^{\mathrm{proj}}
        $
        & $2.0$
        \\

        Upright reference pose
        &
        $\displaystyle
        \exp\!\left(
        -\frac{
        \frac{1}{N_q}
        \|\mathbf q-\mathbf q^{\mathrm{ref}}\|_2^2
        }{0.25}
        \right)
        \mathbf{1}_{\mathrm{up}}
        $
        & $2.0$
        \\

        Hip regularization
        &
        $\displaystyle
        \sum_{j\in\mathcal{H}}
        |a_j|
        $
        & $-0.2$
        \\

        \bottomrule
    \end{tabularx}
\end{table*}

\subsection{Complete Simulation Results}
\label{ap:all_results}
Figure~\ref{fig:all-results} presents the complete results across all 19 selected tasks in the HumanoidBench.

\begin{figure}
    \centering
    \includegraphics[width=\linewidth]{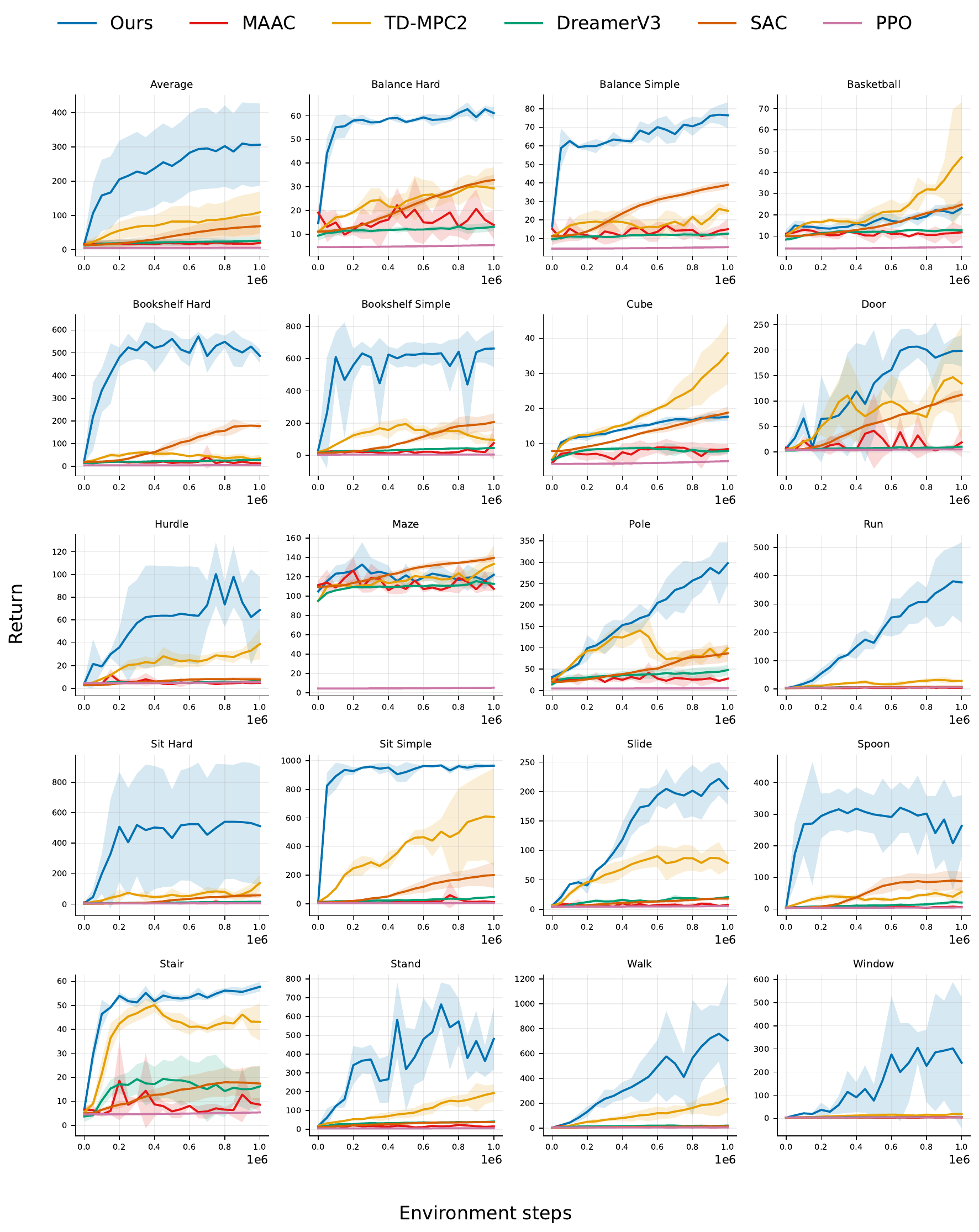}
    \caption{\textbf{All results in Humanoid Benchmark}. The reward curves of 19 tasks in Humanoid Bench. Mean of 3 runs with different seeds. Shaded areas are 95\% confidence intervals.}
    \label{fig:all-results}
\end{figure}

\end{document}